\documentclass[letterpaper, 10 pt, conference]{ieeeconf}
\IEEEoverridecommandlockouts
\usepackage{cite}
\usepackage{amsmath,amssymb,amsfonts}
\usepackage{algorithm,algpseudocode}
\newtheorem{proposition}{Proposition}
\newtheorem{remark}{Remark}
\usepackage{graphicx}
\usepackage{multirow}
\usepackage{multicol}
\usepackage{caption,subcaption}
\DeclareMathOperator*{\argmin}{arg\,min}
\begin{document}    

\title{\LARGE \bf Modular and Parallelizable Multibody Physics Simulation via \\ Subsystem-Based ADMM}

\author{Jeongmin Lee, Minji Lee, and Dongjun Lee
\thanks{
This research was supported by the Industrial Strategic Technology Development Program (20001045) of the Ministry of Trade, Industry \& Energy (MOTIE) of Korea, the Engineering Research Center Program for Soft Robotics (2016R1A5A1938472), and the RS-2022-00144468 of the National Research Foundation (NRF) funded by the Ministry of Science and ICT (MSIT) of Korea. Corresponding author: Dongjun Lee.}
\thanks{The authors are with the Department of Mechanical \& Aerospace Engineering, IAMD and IER, Seoul National University, Seoul, Republic of Korea. \{ljmlgh,mingg8,djlee\}@snu.ac.kr.}
}
\maketitle

\begin{abstract}
In this paper, we present a new multibody physics simulation framework that utilizes the subsystem-based structure and the Alternating Direction Method of Multiplier (ADMM).
The major challenge in simulating complex high degree of freedom systems is a large number of coupled constraints and large-sized matrices. 
To address this challenge, we first split the multibody into several subsystems and reformulate the dynamics equation into a subsystem perspective based on the structure of their interconnection. 
Then we utilize ADMM with our novel subsystem-based variable splitting scheme to solve the equation, which allows parallelizable and modular architecture.
The resulting algorithm is fast, scalable, versatile, and converges well while maintaining solution consistency.
Several illustrative examples are implemented with performance evaluation results showing advantages over other state-of-the-art algorithms.
\end{abstract}

\IEEEpeerreviewmaketitle

\section{Introduction}

Physics simulation enables synthetic data acquisition in a virtual environment to reduce the cost, time, and risk of data-driven methods that are increasingly emerging in robotics \cite{agos19nmi,zeng20tro,ding20icra,lee20tro}.
Further, in terms of finding a solution to the modeled system dynamics equation, it can be directly utilized in various problems such as trajectory optimization \cite{mastalli20crocoddyl}, system identification \cite{carpentier21iden}, etc.
As such, the importance of simulation is increasingly being emphasized, with a plethora of open-source software \cite{RaiSim,bullet,mujoco,flex,sofa}.

One of the most important concerns in robotic simulation research is how to obtain data that is accurate and efficient in terms of computation time.
This is a challenging problem and implies the question of how to formulate the dynamics of systems, and which algorithms to use to solve them.
Since it includes many factors such as discrete-time integration, various types of constraints, friction, system-induced sparsity, numerical algorithms, etc., various methods have been proposed for decades.
However, simulation of a high degree of freedom (DOF) system with many constraints is still a difficult problem \cite{choi21simrobot}.
This is because, fundamentally, all system DOFs are dynamically coupled, so a constraint force acting on a part of the system in general affects the entire system. 
This leads many algorithms to use large-size matrix operations (e.g., factorization) or possibly excessive numerical iterations.

In this paper, we attempt to solve this challenge by developing a novel subsystem-based simulation approach, that is simple, modular, and suitable for parallelization while ensuring the solution consistency and accuracy.    
For this, we first split the multibody system into several subsystems and reformulate the conventional expressions of discrete-time constrained dynamics into a subsystem perspective.
Then inspired by the structure of the Alternating Direction Method of Multiplier (ADMM \cite{boyd11admm}), we present a novel variable splitting scheme and solution process on the reformulated dynamics equation.
This then reduces the solution process to iterations of 1) block-decomposed linear solving of the subsystem dynamics equation (allowing for complete parallelization) and 2) parallel resolution of all the constraint interfaces (with scalar operation only), ensuring low per-iteration computation time and scalability.
Moreover, our method can handle with various types of constraints and also exhibits stable convergence properties, rendering itself as an appealing alternative for robot simulation.
Several multibody simulation examples are then implemented and demonstrated to show the validity of our framework.

The conventional approach to dealing with constrained dynamics equations is applying pivoting algorithms \cite{llyod05icra} after formulating a linear complementarity problem \cite{potra97nd}.
However, since these direct methods require high computational complexity and polygonal friction cone approximation, iteration-based methods have been more widely used in recent studies.
One of the popular approaches is using Gauss-Seidel type iteration per constraint \cite{todorov14convex,macklin14unified,macklin16game,horak19ral}.
These methods scale well for particle-based systems, but not well for systems with generalized coordinate representation (e.g., robot joint angles) and complex internal constraints (e.g., finite element).
Several researches tackle this issue \cite{otaduy09cgf,daviet20tog,carpentier21rss} by taking an operator splitting type method.
However, their applicability to rigid-deformable objects with various constraint types is limited and they still have to deal with the full system size matrices.
Another direction is to apply a Newton-type iteration over the cost including the constraint \cite{macklin19tog,li20ipc,castro22convex}.
Despite their good convergence property, their second-order nature could be problematic for large-sized problems as they require multiple linear problem resolutions.

Our subsystem-based ADMM algorithm may be regarded as an opportunistic compromise between the two directions described above.
By properly separating primal-dual relationships based on subsystems, we circumvent the burdens of handling both with many constraints and large-sized matrices.
In this context, \cite{periet19tog,lee21icra,lee2021real} share some conceptual similarities with our framework proposed here. 
However, their applicability is much limited as compared to our framework, since 1) they need factorization to construct coupling interface equation, which is costly especially as the size of the subsystems grows, and 2) their constructed coupling dynamics is dense, therefore only a small number of inter-connection between subsystems is permitted for reasonable performance.
In contrast, by utilizing the structural peculiarity of ADMM, our proposed framework can handle all the constraints in a decoupled manner for each iteration phase, thereby not only substantially improving the algorithmic efficiency but also allowing for its extension to a wide range of multibody systems. 
We also note that \cite{daviet20tog,tasora21admm,overby17tvcg} employ ADMM structure in simulation. 
However, their full system level approaches still require large-sized matrix operations.     
In contrast, our subsystem-based variable splitting gives a rise to small-sized and parallelized structures, making our scheme much more efficient and scalable.

The rest of the paper is organized as follows. 
Some background materials about constrained dynamics simulation and ADMM will be explained in Sec.~\ref{sec-preliminary}.
Then our simulation framework using subsystem-based ADMM will be described in Sec.~\ref{sec-madmm}.
Some illustrative examples and performance evaluation will be presented in Sec.~\ref{sec-evaluation}.
Finally, discussions and concluding remarks are given in \ref{sec-conclusion}.

\section{Preliminary} \label{sec-preliminary}

\subsection{Constrained Dynamics}
Consider following continuous-time dynamics:
\begin{equation} 
    \begin{aligned}
        M(q)\ddot{q} + C(q,\dot{q})\dot{q} + d\psi^T = f + J(q)^T\lambda
    \end{aligned}
\end{equation}
where $q\in\mathbb{R}^n$ is the generalized coordinate variable of system, $M(q),C(q,\dot{q})\in\mathbb{R}^{n\times n}$ are the mass, Coriolis matrix, $d\psi^T\in\mathbb{R}^n$ is the potential action, $f\in\mathbb{R}^n$ is the external force, and $\lambda\in\mathbb{R}^{n_c}, J(q)\in\mathbb{R}^{n_c\times n}$ are the constraint impulse and Jacobian with $n,n_c$ being the system/constraint dimension.
The discretized version of the dynamics is
\begin{equation} \label{eq-ddyn}
    \begin{aligned}
        &M_k\frac{v_{k+1}-v_k}{t_k} + C_k v_k + d\psi_k^T = f_k + J_{k}^T\lambda_{k} \\
        &\hat{v}_k=\frac{v_k+v_{k+1}}{2}, \quad q_{k+1} \leftarrow \text{update}(q_k, \hat{v}_k, t_k)
    \end{aligned}
\end{equation}
where $k$ denotes the time step index, $M_k=M(q_k)$, $C_k=C(q_k,v_k)$, $t_k$ is the step size, and $v_k,\hat{v}_k\in\mathbb{R}^n$ are the current, representative velocity \cite{kim17ijrr} of each time step. Although we use the midpoint rule here, it can be transformed into other integration rules. From now on, time step index $k$ will be omitted for simplicity but note that all components are still time(step)-varying.

In this paper, we deal with the constraints at the velocity level as in many other works \cite{RaiSim,bullet,mujoco}, which is stable but is based on linearization.
Issues that may arise from linearization can be suppressed by adopting multiple-linearization as in \cite{daviet20tog} or re-linearization \cite{verschoor19collision}, and these will be integrated into our future implementation.
We classify the system constraints into three categories: soft, hard, and contact constraints:

\subsubsection{Soft constraint}
Soft constraints are originated from the elastic potential energy of the system (e.g., finite element).
If the $j$-th constraint is soft, impulse can be written as
\begin{align} \label{eq-scon}
    \lambda_j = -k_j (e_j + \alpha_j J_j\hat{v})
\end{align}
where $e_j\in\mathbb{R}$ and $J_j\in\mathbb{R}^{1\times n}$ are the ($t$-scaled) error and Jacobian for soft constraint, $k_j$ is the gain parameter, and $\alpha_j > 0$ is the variable that includes an implicit term with constraint-space damping. 
The value of $\alpha_j$ is associated with system energy behavior, see \cite{andrews17cgf,kim17ijrr} for more details.

\subsubsection{Hard constraint}
Hard constraints ensure that equations and inequalities for the system are strictly satisfied (e.g., joint limit), including holonomic and non-holonomic types. If the $j$-th constraint is hard, it has the form of
\begin{align} \label{eq-hcon}
    J_j\hat{v} + e_j \ge 0
\end{align}
where $e_j\in\mathbb{R}$ and $J_j\in\mathbb{R}^{1\times n}$ denote the error and Jacobian for hard constraint.
Here, the error can be determined by methods such as Baumgarte stabilization \cite{baumgarte72cm}.

\subsubsection{Contact constraint}
Contact condition is typically the most demanding type since it includes non-linear complementarity relation between primal (i.e., velocity) and dual (i.e., impulse) variables.
We take Signorini-Coulomb condition \cite{lee2022large}, which is the most universal expression for frictional contact.
If the $j$-th constraint is contact, the relation is
\begin{equation} \label{eq-scc}
    \begin{aligned}
    & 0 \le \lambda_{j,n} \perp J_{j,n}\hat{v} + e_{j,n} \ge 0 \\
    & 0 \le \delta_j \perp \mu_j\lambda_{j,n} - \| \lambda_{j,t} \| \ge 0 \\
    &\delta_j \lambda_{j,t} + \mu_j\lambda_{j,n} J_{j,t}\hat{v} = 0
    \end{aligned}
\end{equation}
where $\perp$ denotes complementarity, $e_{j,n}\in\mathbb{R}$ and $J_{j,n}\in\mathbb{R}^{1\times n}$ denote the error and Jacobian for contact normal, $J_{j,t}\in\mathbb{R}^{2\times n}$ is the Jacobian for contact tangential, and $\mu_j$ is the friction coefficient and $\delta_j$ is the auxiliary variable.
There are three situations induced by the condition - open ($\lambda_{j,n}=0$), stick ($\lambda_{j,n}>0, \delta_j=0$), and slip ($\lambda_{j,n}>0, \delta_j>0$). 

\subsection{Alternating Direction Method of Multiplier}

Alternating direction method of multiplier (ADMM \cite{boyd11admm}) is the method to solve the following optimization problem:
\begin{align*}
    \min_{x,z} f(x) + g(z) \quad \text{s.t.} \quad Px+Qz = r
\end{align*}
Based on the augmented Lagrangian defined as,
\begin{align*} 
    &\mathcal{L} = f(x) + g(z) + u^T(Px+Qz-r) + \frac{\beta}{2} \| Px+Qz - r\|^2
\end{align*}
where $u$ is the Lagrange multiplier and $\beta > 0$ is the penalty weight. ADMM iteratively performs alternating minimization of $\mathcal{L}$ with respect to each variable. The iteration process of ADMM can be summarized as follow:
\begin{align*}
    &x^{l+1} = \argmin_{x} \left( f(x) + \frac{\beta}{2} \| Px+Qz^l - r + \frac{1}{\beta}u^l\|^2 \right ) \\
    &z^{l+1} = \argmin_{z} \left( g(z) + \frac{\beta}{2} \| Px^{l+1}+Qz - r + \frac{1}{\beta}u^l\|^2 \right) \\
    &u^{l+1} = u^l + \beta(Px^{l+1}+Qz^{l+1}-r)
\end{align*}
where $l$ is the loop index. 
ADMM is known as robust, simple to implement, and able to attain independent resolution with respect to each variable \cite{boyd11admm,wang19kdd}.

\section{Simulation via Subsystem-Based ADMM} \label{sec-madmm}

\subsection{Subsystem Division} \label{subsec-division}


\begin{figure}[t] 
\centering    
    \begin{subfigure}{3.8cm} 
    \captionsetup{belowskip=2pt}
    \includegraphics[width=3.8cm]{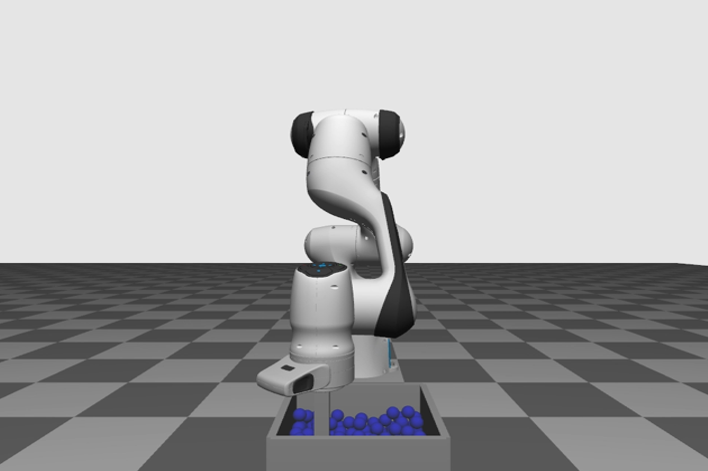} 
    \caption{Granular objects stirring}
    \label{fig-motstir}
    \end{subfigure}
    \begin{subfigure}{3.8cm}
    \captionsetup{belowskip=2pt}
    \includegraphics[width=3.8cm]{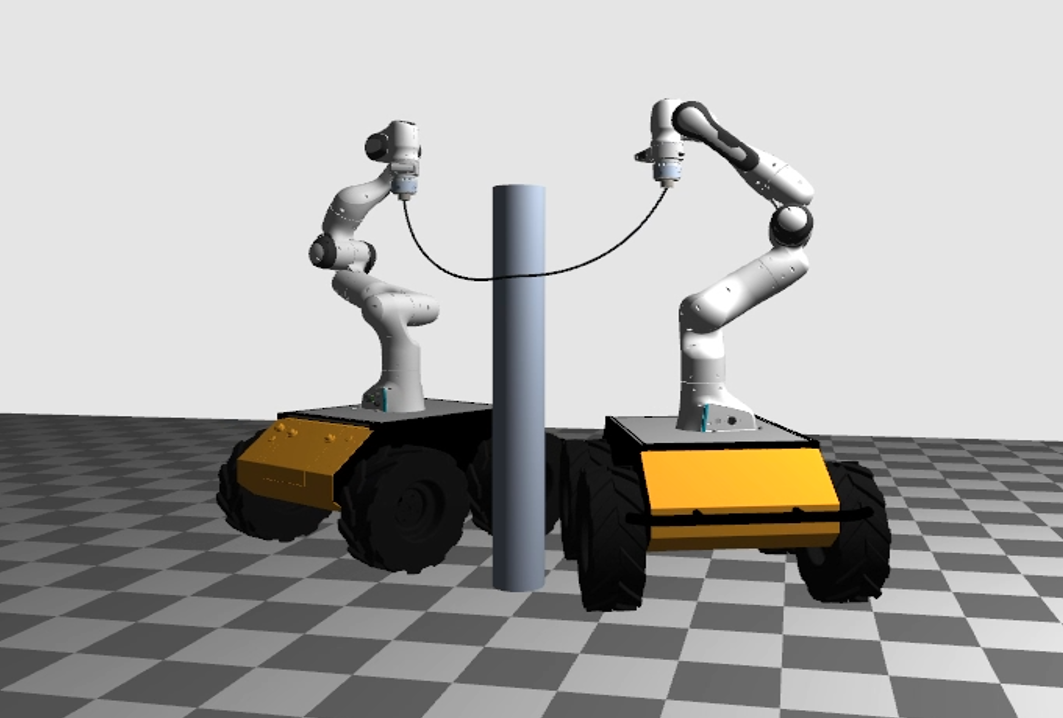} 
    \caption{Cable mobile manipulation}
    \label{fig-motcobot}
    \end{subfigure}
    \begin{subfigure}{3.8cm}
    \includegraphics[width=3.8cm]{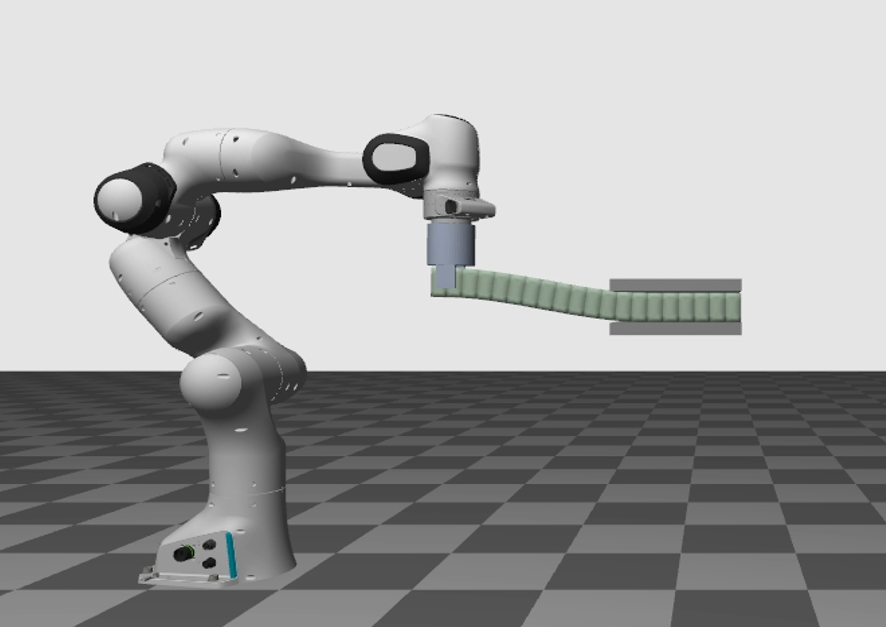} 
    \caption{Deformable body insertion}
    \label{fig-motbeam}
    \end{subfigure}
    \caption{Motivating examples and implementations of our subsystem-based ADMM framework.}
    \label{fig-motivation}
\end{figure}

Our approach starts by dividing the entire system into several subsystems.
See Fig.~\ref{fig-motivation} for our motivational examples.
We assume that objects in typical robotics simulation can be broadly classified into three main classes: rigid body, deformable body, and robot manipulator.
In many cases, each rigid body and manipulator is treated as a single subsystem (as in Fig.~\ref{fig-motstir}).
This is intuitive and allows for preserving modularity for each class (e.g., constant 6 DOF inertia for a rigid body, articulated structure of manipulator).
However, for the situations in which a large number of rigid bodies are connected through soft coupling (e.g., cable modeling as in Fig.~\ref{fig-motcobot}), we find that defining a new subsystem by assembling several rigid body instances can give better performance. 
In the case of a deformable body, its dimension is often so high to conveniently treat it as a single subsystem and causes an imbalance with other objects.
Thus, we split each deformable object into several pieces and consider each as a subsystem, while they are jointly connected using hard constraints (as in Fig.~\ref{fig-motbeam}).

\subsection{Subsystem-Based Dynamics Reformulation} \label{subsec-splitting}

Now consider that the whole system is divided as described in Sec.~\ref{subsec-division}.
If all the subsystems are completely independent (i.e., no coupling), we can formulate each subsystem dynamics using the structure of \eqref{eq-ddyn} and write in the following compressed form:
\begin{align} \label{eq-sdyn}
    &A_i\hat{v}_i = b_i + J_{in,i}^T\lambda_{in,i}
\end{align}
for $i=\left\{1,\cdots,N\right\}$ where $N$ is the number of subsystem, $A_i\in\mathbb{R}^{n_i\times n_i}, b_i\in\mathbb{R}^{n_i}$ are the subsystem dynamics matrices/vectors, and $\lambda_{in,i}\in\mathbb{R}^{n_{in,i}}, J_{in,i}\in\mathbb{R}^{n_{in,i}\times n_i}$ are the \textit{intra}-subsystem constraint impulse/Jacobian while $n_i,n_{in,i}$ are the dimension of subsystem/intra-subsystem constraint.
Here, each $A_i$ is a symmetric positive definite from the mass matrix and energy Hessian approximation \cite{macklin19tog,daviet20tog,lee2022large}.

\begin{remark}
Since \eqref{eq-scon} is in closed-form of $\hat{v}$, it can be directly included in $A_i,b_i$, or still be remained in $\lambda_{in,i}$ of \eqref{eq-sdyn}. 
Currently, this is optional, as both these schemes work fine in our framework.
\end{remark}

Now to take into account the \textit{coupling} constraints between the subsystems, we must add a coupling impulse and the dynamics of the entire system can be written as
\begin{equation} \label{eq-sdyn-total1}
    \begin{bmatrix}
    A_1 & & \\
    & \ddots & \\
    & & A_{N}
    \end{bmatrix}
    \begin{bmatrix}
    \hat{v}_1 \\ \vdots \\ \hat{v}_{N}
    \end{bmatrix} = 
    \begin{bmatrix}
    b_1 \\ \vdots \\ b_{N}
    \end{bmatrix} + 
    \begin{bmatrix}
    J_{in,1}^T\lambda_{in,1} \\ \vdots \\ J_{in,N}^T\lambda_{in,N}
    \end{bmatrix} + J_{cp}^T\lambda_{cp}
\end{equation}
where $\lambda_{cp}\in\mathbb{R}^{n_{cp}}$ and $J_{cp}\in\mathbb{R}^{n_{cp}\times n}$ are the inter-subsystem coupling impulse and Jacobian. Then \eqref{eq-sdyn-total1} can be rewritten as
\begin{align} \label{eq-sdyn-total2}
    A \hat{v} = b + J_{in}^T\lambda_{in} + J_{cp}^T\lambda_{cp}
\end{align}
Note that this new subsystem-based dynamics formulation \eqref{eq-sdyn-total2} does not relax any physical condition, while still allowing to utilize the block-diagonal structure of $A$, even for complex multibody scenarios.

\subsection{ADMM-Based Solver}
To solve \eqref{eq-sdyn-total2} using ADMM, we start by defining the following function:
\begin{align} \label{eq-admm-f}
f_i(\hat{v}_i,x_i) = \frac{1}{2} \hat{v}_i^T A_i \hat{v}_i - b_i^T\hat{v}_i + \mathcal{I}(J_{c,i}\hat{v}_i=x_i) 
\end{align}
where $x_i\in\mathbb{R}^{n_{c,i}}$ is the auxiliary variable, $\mathcal{I}$ is the indicator function, and $J_{c,i}\in\mathbb{R}^{n_{c,i}\times n_i}$ is the row stack of $J_{in,i}$ and $J_{cp,i}$ while $n_{c,i}$ is the summation of intra- and inter-subsystem constraint dimension.
The function \eqref{eq-admm-f} is defined independently for each subsystem and includes the cost for the dynamics ($A_i,b_i$) and the mapping into the constraint space ($J_i\hat{v}_i=x_i$), but does not yet concern with constraint satisfaction.
For the constraint satisfaction, we define the following function:
\begin{align} \label{eq-admm-g}
g(z) = g(z_1,z_2,\cdots,z_{n_s}) = \sum_{j=1}^{n_{in}+n_{cp}} g_j
\end{align}
where each $z_i\in\mathbb{R}^{n_{c,i}}$ is actually interpreted as a duplicated variable of $x_i$ for the $g$ function to enforce the constraints.
The function $g$ can be better understood in constraint-wise, i.e., summation of $g_j$ where $j$ index denotes each constraint.
Each $g_j$ is a function of only the variables corresponding to the $j$-th constraint i.e.,
\begin{align*}
    \{ z_{i,j} ~\vert~  i\in \mathcal{S}_j \}
\end{align*}
where $z_{i,j}$ is the segment of $z_i$ corresponds to the $j$-th constraint and $\mathcal{S}_j$ is the set of subsystem indexes related to the $j$-th constraint.
For the intra-subsystem constraint, the cardinality of $S_j$ (i.e., $\vert \mathcal{S}_j \vert$) is $1$; if the constraint is inter-subsystem coupling, then $\vert \mathcal{S}_j \vert \ge 2$.
Based on the functions \eqref{eq-admm-f} and \eqref{eq-admm-g} defined above, solving \eqref{eq-sdyn-total2} can be reformulated as the following optimization problem:
\begin{equation} \label{eq-admm-sim}
\begin{aligned}
    \min_{\hat{v},x,z}
    &\sum_{i=1}^{n_s}f_i(\hat{v}_i,x_i) + g(z)\\
    \text{s.t.}& \quad x = z
\end{aligned}
\end{equation}
Now applying ADMM iteration on \eqref{eq-admm-sim}, we can obtain the following iteration sequence:
\begin{align}
    &\hat{v}_{i}^{l+1},x_i^{l+1} = \argmin_{\hat{v}_i,x_i} \left( f_i + \frac{\beta_i}{2} \| x_i-z_i^l + \frac{1}{\beta_i}u_i^l\|^2 \right ) \label{eq-admm-1} \\
    &z^{l+1} = \argmin_{z} \left( g + \sum_{i} \frac{\beta_i}{2} \| x_i^{l+1} - z_i + \frac{1}{\beta_i}u_i^l\|^2 \right) \label{eq-admm-2} \\
    &u_i^{l+1} = u_i^l + \beta_i(x_i^{l+1}-z_i^{l+1}) \label{eq-admm-3}
\end{align}
where \eqref{eq-admm-1} and \eqref{eq-admm-3} are actually computed $\forall i$ in parallel and the weight parameter $\beta_i\in\mathbb{R}$ is utilized for each subsystem for better numerical conditions (see also Sec.~\ref{subsection-choicebeta}).
Note that the fixed-point of above iteration will satisfy $\forall i~J_{c,i}\hat{v}_i=x_i=z_i$, therefore it will exactly satisfy \eqref{eq-sdyn-total2} and all constraints (i.e., \eqref{eq-scon}, \eqref{eq-hcon}, \eqref{eq-scc} $\forall j$) without any relaxation.
Since Lagrange multiplier update \eqref{eq-admm-3} is a trivial step, the main consideration here is how to solve \eqref{eq-admm-1} and \eqref{eq-admm-2} in an efficient manner.

\subsubsection{Solving \eqref{eq-admm-1}}

By using an auxiliary variable $x_i$, it can be seen that the dimension of the problem \eqref{eq-admm-1} is expanded to $\text{dim}(\hat{v}_i)+\text{dim}(x_i)$ from the original subsystem dimension $\text{dim}(\hat{v}_i)$.
Consider the following KKT conditions of \eqref{eq-admm-1}:
\begin{align*}
    &A_i\hat{v}_i^{l+1} = b_i + J_{c,i}^T\gamma \\ 
    &\beta_i x_i^{l+1} = \beta_i z_i^l - u_i^l - \gamma \\
    &J_{c,i}\hat{v}_i^{l+1} = x_i^{l+1}
\end{align*}
where $\gamma$ is the Lagrange multiplier. Here, combining these three equations, we can obtain $\hat{v}_i$ by solving the following linear equation:
\begin{align} \label{eq-admm-1i}
    (A_i+\beta_i J_{c,i}^T J_{c,i})\hat{v}_i^{l+1} = b_i + J_{c,i}^T(\beta_i z_i^l - u_i^l)
\end{align}
where the equation is always solvable from the positive definite property of the left-most matrix.
By this procedure, the problem size can be brought back to $\text{dim}(\hat{v}_i)$, therefore the concern about increased computation time due to the inclusion of $x_i$ can be obliviated.
Note that this trick is not possible if we attempt to directly solve the minimization of non-smooth function $f_i$. This rather becomes possible as \eqref{eq-admm-1} in ADMM procedure uses the quadratic augmented term with scalar weight.
In conclusion, the process for solving \eqref{eq-admm-1} is simply obtaining a \textit{subsystem size} linear solution for each subsystem in \textit{parallel}.

\subsubsection{Solving \eqref{eq-admm-2}} \label{subsubsec-admm-2}

As described earlier, $g$ is the summation of all the $g_j$ defined for each constraint.
Accordingly, the problem \eqref{eq-admm-2} can be independently decomposed according to all the constraints as
\begin{align} \label{eq-admm-2j}
    \min_{\underset{i\in \mathcal{S}_j}{z_{i,j}}} 
    \left( g_j + \sum_{i\in \mathcal{S}_j} \frac{\beta_i}{2} \| x_{i,j}^{l+1} - z_{i,j} + \frac{1}{\beta_i}u_{i,j}^l \|^2 \right )
\end{align}
therefore can be solved $\forall j$ in parallel.
Now consider solving \eqref{eq-admm-2j} for bilateral case (i.e., $\vert \mathcal{S}_j \vert = 2$), which is one of the most frequently appearing in practice.
For simplicity, let us assume $\mathcal{S}_j= \left\{ 1,2 \right\}$.

\textbf{Hard constraint:} 
As $z_i$ is the value already mapped into constraint space, $g_j$ only needs to enforce the constraint on $z_{1,j}+z_{2,j}$.
So in the case of hard constraint,
\begin{align} \label{eq-gj-hard}
    g_j = \mathcal{I}(z_{1,j}+z_{2,j}+e_j\ge 0)
\end{align}
and \eqref{eq-gj-hard} can be interpreted as constraint impulse $\lambda_j$ acting on the linear solution of the quadratic terms in \eqref{eq-admm-2j} i.e.,
\begin{align} \label{eq-surrogate}
\begin{split}
    &\beta_1 z_{1,j} = \underbrace{\beta_1 x_{1,j}^{l+1} + u_{1,j}^l}_{:=y_{1,j}^{l+1}} + \lambda_j \\
    &\beta_2 z_{2,j} = \underbrace{\beta_2 x_{2,j}^{l+1} + u_{2,j}^l}_{:=y_{2,j}^{l+1}} + \lambda_j
\end{split}
\end{align}
where we introduce the new variable $y$ for conciseness.
We can see from the structure of \eqref{eq-admm-2j} that the relation \eqref{eq-surrogate} is \textit{matrix-free}, and only consists of scalar weights. 
Thanks to this property, $\lambda_j$ can be computed in a very simple manner
as we combine \eqref{eq-surrogate} with the following complementarity condition:
\begin{align} \label{eq-hard12}
\begin{split}
    &0 \le \lambda_j \perp z_{1,j} + z_{2,j} + e_j \ge 0
\end{split}
\end{align}
the solution for $\lambda_j$ can be obtained with the simple scalar operation:
\begin{align*}
    &\lambda_j = \Pi_{\ge 0}\left (-\frac{\beta_1^{-1}y_{1,j}^{l+1} + \beta_2^{-1} y_{2,j}^{l+1} + e_j}{\beta_1^{-1}+\beta_2^{-1}} \right )
\end{align*}
where $\Pi_{\ge 0}$ denotes the projection on positive set.

The matrix-free relation \eqref{eq-surrogate} is the same for other types of constraints (soft, contact), while \eqref{eq-hard12} to be replaced with other relation.

\textbf{Soft constraint:} 
From the structure of \eqref{eq-scon},
\begin{align} \label{eq-soft12}
    \lambda_j = -k_j(e_j+\alpha_j (z_{1,j}+ z_{2,j}))
\end{align}
has to be satisfied. 
Then by substituting \eqref{eq-surrogate} to \eqref{eq-soft12}, we can obtain the impulse solution as
\begin{align*}
    \lambda_j = -\frac{k_j(e_j + \alpha_j(\beta_1^{-1} y^{l+1}_{1,j} + \beta_2^{-1} y^{l+1}_{2,j})}{(1+(\beta_1^{-1}+\beta_2^{-1})\alpha_j k_j)}
\end{align*}
which is also very simple to compute.

\textbf{Contact constraint:} 
Here the relation between $z_{1,j}+z_{2,j}$ and $\lambda_j$ must follow \eqref{eq-scc}, therefore
\begin{align} \label{eq-contact12}
    \begin{split}
    & 0 \le \lambda_{j,n} \perp z_{1,j,n}+z_{2,j,n} + e_{j,n} \ge 0 \\
    & 0 \le \delta_j \perp \mu\lambda_{j,n} - \| \lambda_{j,t} \| \ge 0 \\
    &\delta \lambda_{j,t} + \mu\lambda_{j,n} (z_{1,j,t}+z_{2,j,t}) = 0
    \end{split}
\end{align}
Despite the complexity of \eqref{eq-contact12}, solution can be easily obtained from the simple scalar structure of \eqref{eq-surrogate}:
\begin{align*}
    \lambda_j = \Pi_{\mathcal{C}}\left (-\frac{\beta_1^{-1}y_{1,j}^{l+1} + \beta_2^{-1} y_{2,j}^{l+1} + e_j}{\beta_1^{-1}+\beta_2^{-1}} \right )
\end{align*}
where $\Pi_{\mathcal{C}}$ denotes the projection on the friction cone.
The process can be done through a few algebraic operations, while respecting all contact conditions \cite{lee2022large}.

Although we explain the process only for the bilateral case, it can be shown straightforwardly that such a simple solution form can be derived for other cases as well.

\begin{algorithm} [t]
\caption{Simulation via Subsystem-Based ADMM} 
\label{alg1}
\begin{algorithmic}[1] 
\State Subsystem division for given multibody (Sec.~\ref{subsec-division})
\While{simulation}
\State $\forall i$ construct $A_i,b_i$ in parallel
\State $\forall j$ construct $e_j,J_j$ in parallel
\State $\forall i$ factorize $A_i+\beta_iJ_{c,i}^TJ_{c,i}$ in parallel
\While{loop}
\State $\forall i$ update $\hat{v}_{i}^{l+1}$ from \eqref{eq-admm-1i} in parallel
\State compute residual $\theta$ from \eqref{eq-residual}
\If{$\theta < \theta_{th}$ or $l=l_{max}$} 
\State \textbf{break}
\EndIf
\State $\forall j$ update $z_j^{l+1}$ from \eqref{eq-admm-2j} in parallel
\State $\forall i$ update $u^{l+1}$ from \eqref{eq-admm-3} in parallel
\State $l\leftarrow l+1$
\EndWhile
\State{update each subsystem state using $\hat{v}_i^{l+1}$}
\EndWhile
\end{algorithmic}
\end{algorithm}

\subsection{Convergence}

It can be easily verified that each $f_i$ and $g_j$ for hard and soft constraints is convex in our formulation \eqref{eq-admm-sim}.
For contact conditions, $g_j$ may not be convex, but can be convexified by adopting the relaxed convex model \cite{todorov14convex}. 
In such cases, our method can guarantee convergence \cite{boyd11admm}.
Although we have not encountered the convergence issue associated with non-convexity of \eqref{eq-scc}, a more thorough analysis will be left for future work.

\subsubsection{Residual}
Originally, our process \eqref{eq-admm-1}, \eqref{eq-admm-2}, \eqref{eq-admm-3} is the iteration of $(\hat{v},x,z,u)$ and both primal and dual residual \cite{boyd11admm} are required to check the condition to terminate the iteration.
Instead, for our framework, we use the variable $y$ in \eqref{eq-surrogate} to define the residual as
\begin{align} \label{eq-residual}
    \theta = \sum_{i=1}^{n_s} \| y_i^{l+1}-y_i^l \|^2
\end{align}
where $\theta$ is the residual value.
This means that the iteration can be reinterpreted in terms of the lower-dimensional variable $y$, and the process of calculating the residuals can be more concise.
The following proposition provides the rationale of the statement:
\begin{proposition}
$(\hat{v}^{l+1},x^{l+1},z^{l+1},u^{l+1})$ is the fixed-point of the iteration \eqref{eq-admm-1}, \eqref{eq-admm-2} and \eqref{eq-admm-3}, if and only if $\theta = 0$.  
\end{proposition}
\begin{proof}
$\left( \Rightarrow \right)$ This is trivial. $\left( \Leftarrow \right)$ As $\theta=0$ denotes $\forall y_i^{l+1}=y_i^l$, we can find that $z^{l+1}=z^l$ holds as $\forall\lambda_j$ are uniquely determined from $y$.
Then as \eqref{eq-admm-3} is equivalent to $u_i^{l+1}=y_i^{l+1}-z_i^{l+1}$, $u^{l+1}=u^l$ also holds.
Finally, $\hat{v}$ is determined from $z$ and $u$ \eqref{eq-admm-1i}, so we can conclude that the set value is in fixed-point of the iteration.
\end{proof}

\subsubsection{Choice of $\beta$} \label{subsection-choicebeta}

We find that iteration has stable convergence regardless of $\beta$, but the value of $\beta$ affects the convergence rate.
We empirically confirm that the following $\beta$ setting exhibits good performance:
\begin{align}
    \forall\beta_i=\frac{\text{Tr}\left ( A_i \right )}{\text{Tr}\left (J_{c,i}^T J_{c,i}\right )}
\end{align}
which suggests a balanced weight between dynamics-related term $A_i$ and constraint-related term $J_{c,i}^T J_{c,i}$.
A more in-depth theoretical analysis of the strategy will be discussed in future work.

\subsection{Summary}

Our physics simulation framework via subsystem-based ADMM is summarized in Alg.~1.
As described earlier, the major part of the procedure is \textit{subsystem-wise} parallel solving of \eqref{eq-admm-1i} (line 7) and \textit{constraint-wise} parallel solving of \eqref{eq-admm-2j}  (line 12). 
From these characteristics, the computational complexity of our algorithm is at least linear: $\mathcal{O}(n_s+n_{in}+n_{cp})$. 
If parallelization is taken into account, it will be lower.

\section{Examples and Evaluations} \label{sec-evaluation}

We use an Intel Core i7-8565 CPU 1.80GHz (Quad-Core), OpenGL as a rendering tool, C++ Eigen as a matrix computation library, and C++ OpenMP as a parallelization library in our implementation.
Time step is set to $10~\rm{ms}$ for all examples.
See also our supplemental video.

\subsection{Scenarios}
We implement three high-DOF multibody manipulation scenarios.
In general, they consist of a combination of high-gain controlled robotic arms and lightweight objects with multi-type constraints, resulting in numerically challenging situations.
We employ Franka Emika panda \cite{franka} as a robot arm and Husky \cite{husky} as a ground vehicle.

\subsubsection{Granular object stirring}
The example is illustrated in Fig.~\ref{fig-motstir}: the robot arm uses an end effector to stir the granular material contained in the box.
The granular material consists of a total of $216$ spheres with a radius of $1~\rm{cm}$ and a weight of $4~\rm{g}$.
The total system dimension is $1303$, and since each rigid body and robot is treated as a subsystem, there are a total of $217$ subsystems.

\subsubsection{Collaborative cable manipulation}
The example is illustrated in Fig.~\ref{fig-motcobot}: two mobile manipulator consist of a ground vehicle and a robot arm are transporting and winding a flexible cable.
Cable is modeled by $640$ rigid bodies and soft constraint from Cosserat model, with length $1.2~\rm{m}$, diameter $8~\rm{mm}$, Young modulus $0.1~\rm{MPa}$, and Poisson ratio $0.49$.
Each mobile manipulator is modeled as $10$-DOF system while its movement constrained by non-holonomic constraint (no-slip).
Total system dimension is $3840$, and we treat each mobile manipulator and $4$ cable segments as a subsystem, making a total of $162$ subsystems.

\subsubsection{FEM beam insertion}
The example is illustrated in Fig.~\ref{fig-motbeam}: the robot arm inserts the deformable beam modeled with co-rotational FEM through narrow gap.
The size of beam is $0.05\times0.05\times0.5~\rm{m}$, with a Young modulus $10~\rm{MPa}$ and a Poisson ratio $0.45$.
The FEM model consists of $1591$ nodes and $6347$ tetrahedral elements, therefore total dimension is $4780$.
We divide the model into $20$ subsystems so the entire system consists of a total of $21$ subsystems including the manipulator.

\begin{table*}[t]
\small
\centering
\renewcommand{\arraystretch}{1.3}{
\resizebox{17.7cm}{!}{
\begin{tabular}{|c|c|c|c|c|c|c|c|c|c|c|c|c|c|c|c|c|}
\hline
\multicolumn{2}{|c|}{Solver} & \multicolumn{3}{c|}{PGS} & \multicolumn{3}{c|}{PJ} & \multicolumn{3}{c|}{FADMM} & \multicolumn{3}{c|}{NNewton} & \multicolumn{3}{c|}{\textbf{SubADMM}} \\
\hline
\multicolumn{2}{|c|}{Iteration} &
30 & 60 & 90 & 
30 & 60 & 90 & 
30 & 60 & 90 & 
3 & 6 & 9 & 
30 & 60 & 90  \\
\hline
\multirow{2}{*}{Stir} & AT 
& $14.50$ & $23.91$ & $32.27$ 
& $3.235$ & $5.688$ & $9.425$
& $28.41$ & $41.40$ & $56.95$ 
& $24.35$ & $46.99$ & $75.44$ 
& $3.489$ & $5.940$ & $8.705$ \\
\cline{2-17} & AA
& $4.427$ & $4.928$ & $5.248$
& $3.033$ & $3.349$ & $3.562$
& $4.107$ & $4.644$ & $5.009$
& $3.565$ & $4.429$ & $5.324$ 
& $4.069$ & $4.579$ & $5.023$ \\
\hline
\multirow{2}{*}{Cable} & AT 
& $48.08$ & $59.30$ & $72.96$ 
& - & - & -   
& $16.74$ & $23.52$ & $32.96$ 
& $43.35$ & $87.98$ & $132.7$ 
& $2.288$ & $4.285$ & $6.402$ \\ 
\cline{2-17} & AA 
& $4.141$ & $4.860$ & $5.404$ 
& - & - & -  
& $4.189$ & $4.634$ & $4.910$ 
& $3.270$ & $4.454$ & $5.278$  
& $4.344$ & $4.984$ & $5.222$ \\ 
\hline 
\multirow{2}{*}{Beam} & AT 
& $231.4$ & $241.2$ & $251.2$ 
& - & - & -
& $50.33$ & $92.11$ & $130.3$
& $188.5$ & $360.2$ & $525.7$
& $13.41$ & $24.67$ & $35.50$  \\
\cline{2-17} & AA 
& $3.895$ & $4.194$ & $4.255$
& - & - & - 
& $4.220$ & $4.478$ & $4.756$ 
& $2.445$ & $3.494$ & $4.945$
& $4.326$ & $4.743$ & $4.925$ \\
\hline 
\end{tabular}
}
}   
\caption{Evaluation results for various solvers. 
AT: average compuatation time (\rm{ms}), AA: average accuracy (constraint error value converted using $-\log(\cdot)$ before averaged, therefore bigger is better).
Unmarked values (-) means that the simulation fails to run successfully (e.g., significant penetration).}
\label{table-result}
\end{table*}

\subsection{Baselines}

We implement the following algorithms for performance comparison, with our method being denoted as SubADMM.

\subsubsection{Projected Gauss-Seidel (PGS)}
PGS is a representative algorithm in robotics and graphics fields \cite{todorov14convex,macklin14unified,macklin16game,horak19ral} and software \cite{mujoco,bullet,flex}. 
We implement an algorithm with conjugate gradient-based acceleration to improve its performance.

\subsubsection{Projected Jacobi (PJ)}
PJ is similar to PGS, but they do not solve constraints sequentially, but rather solve and update them in parallel at once.

\subsubsection{Full ADMM (FADMM)}
State-of-the-art implementations of ADMM algorithms \cite{stellato20osqp} can be used to solve physics simulation, which is specified in \cite{tasora21admm}.
The main difference with our algorithm is that they require solving of the full-system size matrix for each iteration.

\subsubsection{Nonsmooth Newton (NNewton)}
We also implement a recently proposed algorithm that transform the constraints into non-smooth function and solve it using Newton iteration. 
We refer \cite{macklin19tog,andrews22course} for details.

\subsection{Performance Index}

We apply the same number of iteration ($30,60,90$) for all algorithms except NNewton and measure the average solver computation time and constraint error norm from the simulation results.
In the case of NNewton, considering its second-order nature (cost per iteration is high but uses fewer iterations), the number of iterations is reduced by $1/10$ (i.e., $3,6,9$).
Constraint error for contact is calculated using Fischer-Burmeister function \cite{macklin19tog}.

\subsection{Results}

Evaluation results are summarized in Table~\ref{table-result}.
For graunlar object strring scenario, PJ and SubADMM shows the fastest computation speed, and this is due to their structure suitable for parallelization.
However, constraint error of PJ is signifcantly higher than SubADMM.
This reflects the unstable and slow convergence of the Jacobi-style iteration. 
On the other hand, SubADMM shows comparable error with other methods and shows its validity in terms of accuracy.
In the case of the cable and beam scenario, the computation performance of PGS and PJ becomes lower as the Delassus operator assembly is more complicated.
As such, FADMM outperforms them, yet SubADMM is still the fastest.
This is due to our special structure, which, as mentioned earlier, only requires parallelized resolution of the subsystem matrices without dealing with large-sized matrices.
In the similar vein, SubADMM also has an efficiency advantage over NNewton.
Algorithms other than PJ showed valid accuracies, while PJ failed to generate an adequate simulation results.
In summary, the results demonstrate all of the methodologically described advantages of SubADMM: 1) it avoids burdens on both many constraints and large-sized matrices, and 2) it does not use certain approximations on the model and has a good convergence property.

\subsection{Scalability}


\begin{figure}[t] 
\centering    
    \begin{subfigure}{4.0cm} 
    \includegraphics[width=4.0cm]{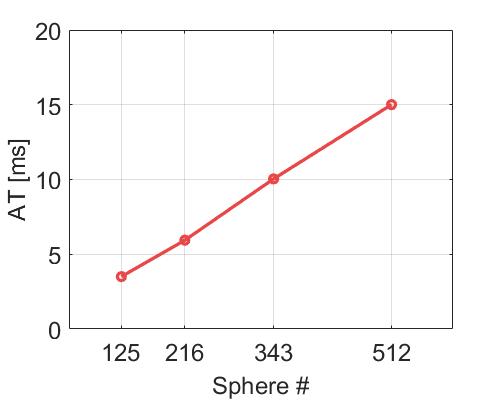} 
    \caption{Stir}
    \end{subfigure}
    \begin{subfigure}{4.0cm}
    \includegraphics[width=4.0cm]{Figure/scalestir.jpg} 
    \caption{Stir}
    \end{subfigure}
    \caption{Scalablity test results of SubADMM.}
    \label{fig-scalablity}
\end{figure}

To precisely evaluate the scalability of our method, we measure the computation time (iteration: $60$) by increasing the number of spheres in the stir scenario and the number of segments in the cable scenario.
Fig.~\ref{fig-scalablity} shows linear complexity of SubADMM (R-squared value: $0.9993$, $0.9998$).

\section{Discussions and Conclusions} \label{sec-conclusion}

In this paper, we present a new physics simulation framework based on subsystem-based ADMM.
Our approaches combines a novel subsystem-based formulation \eqref{eq-sdyn-total1} and operator splitting \eqref{eq-admm-f} and \eqref{eq-admm-g}, thereby achieve parallelizable and modular architecture for general multibody dynamics.
Several examples are implemented and evaluations show the advantages of our framework against state-of-the-art algorithms.
We believe that a generic implementation (similar to the open source form) will make a good contribution to the robotics community.
We also believe that our work can be extended to the area of optimal control by exploiting the coupled structure of the large-size optimization problem (e.g., time correlation).
Finally, similar to typical ADMM, the convergence property of our algorithm is stable but still linear.
Therefore, combination with second-order acceleration schemes will be a promising research direction.

\newpage
\bibliographystyle{unsrt}
\bibliography{reference}

\end{document}